\documentclass[supp]{new_tlp} %%added supp option

\usepackage{times}
\usepackage{helvet}
\usepackage{courier}
\usepackage{xspace}
\usepackage[usenames,dvipsnames]{color}
\usepackage{amsmath}
\usepackage[all]{xy}
\usepackage{supp}%%added
\newcommand{\printsolver}[1]{\textsc{#1}\xspace}
\newcommand{\chuffed}{\printsolver{chuffed}}

\newcommand{\libmzn}{\printsolver{libmzn}}

\newcommand{\false}{\mathit{false}}
\newcommand{\true}{\mathit{true}}

\newcommand{\vars}{\mathit{vars}}

\newcommand{\head}{\mathit{head}}
\newcommand{\mn}{\mathit{min}}
\newcommand{\mx}{\mathit{max}}
\newcommand{\sm}{\mathit{sum}}

\newcommand{\rules}{\mathit{rules}}

\newcommand{\created}{\mathit{cr}}

\newcommand{\ujb}{\mathit{ujb}}

\newcommand{\ground}{\mathit{grnd}}

\newcommand{\gen}{\mathit{gen}}

\newcommand{\constraints}{\mathit{constraints}}
\newcommand{\set}{\mathit{set}}

\newcommand{\where}{\text{ where }}

\newcommand{\VV}{{\cal V}}

\newcommand{\AFV}{{\cal F}}
\newcommand{\ANV}{{\cal S}}

\newcommand{\bif}{\textbf{if}}

\newcommand{\belseif}{\textbf{elif}}
\newcommand{\bfor}{\textbf{for}}
\newcommand{\bwhile}{\textbf{while}}
\newcommand{\breturn}{\textbf{return }}

\DeclareMathOperator{\cupeq}{{\cup\!\!=}}

\DeclareMathSymbol{\naf}{\mathord}{symbols}{"18}

\newcommand{\rehan}[1]{\marginpar{\color{Periwinkle}\sc Rehan}
\textbf{\color{Periwinkle}[#1]}}

\newcommand{\ignore}[1]{}
\newcommand{\cupp}{\ensuremath{\cup}:=~}
\newcommand{\cre}{cr}

\def\inc(#1,#2){#1_{#2}^{\uparrow}}
\def\dec(#1,#2){#1_{#2}^{\downarrow}}
\def\nonmon(#1,#2){#1_{#2}^{\updownarrow}}
\def\scope(#1,#2){#1|_#2}
\def\fdom(#1,#2,#3){#1_{#2}^{#3}}
\def\tpi(#1){T_P \uparrow {#1}}
\def\jb(#1,#2){\mathit{jb}_{#2}({#1})}
\def\push(#1,#2){{#1} \cupp {#2}}
\def\maxd(#1,#2){\mx(#1(#2))}
\def\mind(#1,#2){\mn(#1(#2))}

\newcommand{\minf}{-\infty}

\newcommand{\proj}[2]{\exists_{#1}(#2)}
\newcommand{\red}[2]{{#1}^{#2}}

\newtheorem{theorem}{Theorem}
\newtheorem{example}{Example}

\newcommand{\noproofs}[1]{}

\newcommand{\noexamples}[1]{#1}

\title{
Grounding Bound Founded Answer Set Programs
}
\author[R. A. Aziz and G. Chu and P. J. Stuckey]
	{Rehan Abdul Aziz, Geoffrey Chu and Peter J. Stuckey \\
	National ICT Australia, Victoria Laboratory,\thanks{NICTA is funded by the Australian Government as represented by the
Department of Broadband, Communications and the Digital Economy and the
Australian Research Council through the ICT Centre of Excellence program.}
        \\
	Department of Computing and Information Systems, \\
	University of Melbourne, Australia \\
	Email:~{\tt raziz@student.unimelb.edu.au, gchu@csse.unimelb.edu.au, pjs@csse.unimelb.edu.au}
	}
	
\pubauthor{R.A. Aziz and G. Chu and P. J. Stuckey}%%added

\begin{document}

\label{firstpage}

\maketitle
\begin{abstract}
Bound Founded Answer Set Programming (BFASP) is an extension of 
Answer Set Programming (ASP) that extends stable model
semantics to numeric variables. 
While the theory of BFASP is defined on ground rules, in practice BFASP
programs are written as complex non-ground expressions.
Flattening of BFASP 
is a technique used to simplify arbitrary expressions of 
the language to a small and well defined set of primitive expressions.
In this paper, 
we first show how we can flatten arbitrary BFASP rule expressions,
to give equivalent BFASP programs. 
Next, we extend the 
bottom-up grounding technique and magic set transformation used by ASP 
to BFASP programs. 
Our implementation shows that for BFASP problems, 
these techniques can significantly reduce the ground program size,
and improve subsequent solving.
\end{abstract}
\begin{keywords}
    Answer Set Programming, Grounding, Flattening, Constraint ASP, Magic Sets
\end{keywords}

\section{Introduction}

Many problems in the areas of planning or reasoning can be efficiently
expressed using Answer Set Programming (ASP)~\cite{baral03}. 
ASP enforces stable model semantics~\cite{stable} on the program, which
disallows solutions representing circular reasoning. For example, given only
rules $b \leftarrow a$ and $a \leftarrow b$, the assignment $a = \true, b = \true$ would be a solution under the
logical semantics normally used by Boolean Satisfiability (SAT)~\cite{sat_solver}
solvers or Constraint Programming (CP)~\cite{book} solvers, but would not be a
solution under the stable model semantics used by ASP solvers. 
%Thus ASP is particularly useful in problem domains where circular reasoning needs to be avoided.

Bound Founded Answer Set Programming (BFASP)~\cite{bfasp} is an extension of
ASP to allow \emph{founded} integer and real variables.  
This makes it possible to concisely express and efficiently solve problems involving inductive definitions of numeric
variables where we want to disallow circular reasoning. As an example consider the Road Construction problem (\textsl{RoadCon}).  
We wish to decide which roads to build such that the shortest paths between various cities are acceptable, with the minimal total 
cost. This can be  modeled as:

\newcommand{\Edge}{\mathit{Edge}}
\newcommand{\Node}{\mathit{Node}}
\newcommand{\Demand}{\mathit{Demand}}
\newcommand{\built}{\mathit{built}}
\newcommand{\cost}{\mathit{cost}}
\newcommand{\efrom}{\mathit{from}}
\newcommand{\eto}{\mathit{to}}
\newcommand{\len}{\mathit{len}}
\newcommand{\shp}{\mathit{sp}}
\newcommand{\demfrom}{\mathit{d\_from}}
\newcommand{\demto}{\mathit{d\_to}}
\newcommand{\dem}{\mathit{demand}}

$$
\begin{array}{l}
\text{minimize} \sum_{e \in \Edge} \built[e] \times \cost[e] \\
\forall y  \in \Node: \shp[y,y] \leq 0 \\
\forall y \in \Node, e \in \Edge: \shp[\efrom[e],y] \leq \len[e] + \shp[\eto[e],y] \leftarrow \built[e] \\ 
\forall y \in \Node, e \in \Edge: \shp[\eto[e],y] \leq \len[e] + \shp[\efrom[e],y] \leftarrow \built[e] \\ 
\forall p \in \Demand: \shp[\demfrom[p],\demto[p]] \leq \dem[p] \\
\end{array}
$$
The decisions are which edges $e$ are built ($\built[e]$).
The aim is to minimize the total cost of the edges $\cost[e]$ built.
The first rule is a base case that says that shortest path from a node to itself
is 0.
The second constraint defines the shortest path $\shp[x,y]$ from $x$ to $y$:
the path from $x$ to $y$ is no longer than from $x$ to $z$ along edge $e$ if
it is built plus the shortest path from $z$ to $y$;
and the third constraint is similar for the other direction of the edge.
The last constraint ensures 
that the shortest path for each of a given set
of paths $p \in \Demand$ is 
no longer than its maximal allowed distance $\dem[p]$.
The above model has a trivial solution with cost 0 by
setting $\shp[x,y] = 0$ for all
$x,y$. In order to avoid this, we require that the $\shp$ variables are
(upper-bound) \emph{founded} variables, that is they take the largest
possible justified value.  The first three constraints are actually
\emph{rules} which justify upper bounds on $\shp$, the last constraint is a
restriction that needs to be met and cannot be used to justify upper bounds. 
Solving such a BFASP is challenging, mapping to CP models leads to inefficient
solving, and hence we need a BFASP solver which can reason directly about 
\emph{unfounded sets}~\cite{vangelder} of numeric assumptions.
Note that \emph{Constraint ASP} (CASP) and hybrid systems such
as those given by \cite{mellarkod_ai,clingcon,inca_ng,asp_via_mip,ezcsp,inductive_cp} cannot solve the above problem without grounding
the numeric domain to propositional variables and running into the grounding bottleneck. 
BFASP has been shown to subsume CP, ASP, CASP and Fuzzy ASP~\cite{NieuwenborghCV06,blondeel2013fuzzy}, see~\cite{bfasp} for details.

The above encoding for Road Construction problem is a \emph{non-ground} BFASP since it is
parametric in the data: $\Node$, $\Edge$, $\Demand$, $\cost$, $\efrom$, $\eto$,
$\len$, $\demfrom$, $\demto$ and $\dem$. In this paper we consider how to
efficiently create a ground BFASP from a non-ground BFASP given the data.
This is analogous to \emph{flattening}~\cite{mznfn} of constraint models
and \emph{grounding}~\cite{lparse,gringo,dlvgrounder} of ASP programs.
The contributions of this paper are:
a flattening algorithm that 
transforms complex expressions to primitive forms while preserving the stable
model semantics, 
a generalization of bottom-up grounding for normal logic programs 
to BFASPs and a generalization of the 
magic set transformation \cite{original_magic,power_magic} for normal logic programs to BFASPs.

\section{Preliminaries}

\subsection{Constraints and Answer Set Programming}

%Let $(x_1 \in X_1, \ldots, x_k \in X_k)$ be denoted by $\bar{x \in X}$.
We consider three types of variables: integer, real, and Boolean.
%\pjs{Do we consider reals!}
Let $\VV$ be a set of variables.
%We use $[l,u]$ to indicate the interval $l \leq x \leq u$.
A \emph{domain} $D$ maps each variable $x \in \VV$ to a set of constant values $D(x)$. 
%We use $D_N(x)$ and $\max(D(x))$ as a shorthand to represent minimum and maximum values of a variable $x$ in $D$ respectively.
A \emph{valuation} (or assignment) $\theta$ over variables
$vars(\theta) \subseteq \VV$ maps each variable $x \in vars(\theta)$ to a value
$\theta(x)$.  A restriction of assignment $\theta$ to variables $V$,
$\theta|_V$, is the the assignment $\theta'$ over $V \cap vars(\theta)$ 
where $\theta'(v) = \theta(v)$.
A \emph{constraint} $c$ is a set of assignments over the variables $vars(c)$,
representing the solutions of the constraint. 
A constraint $c$ is 
\emph{monotonically increasing (resp. decreasing)} w.r.t. a variable $y \in \vars(c)$
if for all solutions $\theta$ that satisfy $c$, increasing (resp. decreasing) 
the value of $y$ also creates a solution, that is
$\theta'$ where $\theta'(y) > \theta(y)$ (resp. $\theta'(y) < \theta(y)$), and $\theta'(x) = \theta(x), x \in
vars(c) - \{y\}$, is also a solution of $c$. 
A \emph{constraint program (CP)}  is a collection of variables $\VV$ 
and constraints $C$ on those variables ($vars(c) \subseteq \VV, c \in C$). 
%Note that domain constraints over variables are also included in the constraints. 
A \emph{positive-CP} $P$ is a CP where
each constraint is 
increasing in exactly one variable and decreasing in the rest. 
The \emph{minimal}
solution of a positive-CP is an assignment $\theta$ that satisfies $P$ s.t. there is no
other assignment $\theta'$ that also satisfies $P$ and there exists 
a variable $v$ for which
$\theta'(v) < \theta(v)$. Note that for Booleans, $\true > \false$.
A positive-CP $P$ always has a unique minimal solution. 
%This unique minimal solution $\theta$ is given by: $\forall x, \theta(x) = min\{v \mid P \Rightarrow x \geq v\}$. 
If we have bounds consistent propagators for all the constraints in the program, then this unique minimal solution can be found simply by performing bounds propagation on all constraints until a fixed point is reached, and then setting all variables to their lowest values.

\ignore{ which can be constructed as follows:
start by setting all variables to their minimum values, continue to loop over all constraints and
if a constraint has an increasing variable, increase its value to to the minimum value necessary
to satisfy the constraint, until a fixed point is reached.
\rehan{Please see if this is correct and clear, because it is used in proofs.}}

%\subsection{Answer set programming}

A \emph{normal logic program} $P$ is a collection of \emph{rules} of the form: 
$b_0 \leftarrow b_1 \wedge \ldots \wedge b_n \wedge \neg b'_1 \wedge \ldots \wedge \neg b'_m$
where $\{b_0,b_1,\ldots, b_n,b'_1, \ldots, b'_m\}$ are Boolean variables. 
$b_0$ is the \emph{head} of the rule while the RHS of
the reverse implication is the \emph{body} of the rule. 
A rule without any negative literals is a \emph{positive rule}.
%A program with only positive rules is a HORN-CP.
A \emph{positive program} is a collection of positive rules.  The
\emph{least model} of a positive program is an assignment $\theta$ that
assigns $\true$ to the minimum number of variables.  The \emph{reduct} of $P$
w.r.t. an assignment $\theta$ is written $P^\theta$ and is a positive program
obtained by transforming each
rule $r$ of $P$ as follows: if there exists an $i$ for which $\theta(b'_i) =
\true$, discard the rule, otherwise, discard all negative literals $\{b'_1, \ldots, b'_m\}$ from
the rule.  The stable models of $P$ are all assignments $\theta$ for which
the least model of $P^\theta$ is equal to $\theta$.  Note that if we
consider a logic program as a constraint program, then a positive program is
a positive-CP and the least model of that program is equivalent to the minimal
solution defined above.

\subsection{Bound Founded Answer Set Programs (BFASP)}

BFASP is an extension of ASP that extends its semantics over integer and real variables.
In BFASP, the set of variables 
is a union of two disjoint sets: standard $\ANV$ and \emph{founded} 
variables $\AFV$.\footnote{For the rest of this paper we only consider \emph{lower bound}
  founded variables, analogous to founded Booleans. Upper bound founded
  variables can be implemented as negated lower bound founded
  variables, e.g. replace $sp[x,y]$ in the Road Construction
  example by $-nsp[x,y]$ where $nsp[x,y]$ is lower bound founded. }  
A rule $r$ is a pair $(c,y)$ where $c$ is a constraint, $y \in \AFV$ is the 
head of the rule and it is increasing in $c$.
A bound founded answer set program (BFASP) $P$ is a tuple $(\ANV, \AFV, C, R)$ where $C$ and $R$ are sets of constraints
and rules respectively (also accessed as $\constraints(P)$ and $\rules(P)$ resp.). Given a variable $y \in \AFV$, $\rules(y)$ is the set of rules with $y$ as their heads.
Each standard variable $s$ is associated with a lower and an upper bound, written $lb(s)$ and $ub(s)$ respectively.

The reduct of a BFASP $P$ w.r.t. an assignment $\theta$ is a 
positive-CP made from each rule $r = (c,y)$ by replacing in $c$
every variable $x \in vars(c) - \{y\}$ s.t. $x$ is a standard variable or $c$ is not decreasing in $x$, 
by its value $\theta(x)$ to create a positive-CP constraint $c'$. Let $\red{r}{\theta}$ denote this constraint.
If $\red{r}{\theta}$ is not a tautology, 
it is included in the reduct $\red{P}{\theta}$.
An assignment $\theta$ is a stable solution of $P$ iff i) it satisfies
all the constraints in $P$ and ii) it is the minimal solution
that satisfies $P^\theta$.
For a variable $y \in \AFV$, the
\emph{unconditionally justified bound} of $y$, written $\ujb(y)$, is a value that is unconditionally justified by the rules of the program 
regardless of what the standard variables are fixed to.
E.g. if we have a rule: $(y \geq 3 + x, y)$ where $x$ is a standard variable with 
domain $[0, 10]$, then we can set $\ujb(y) = 3$.
%For a founded variable $y$, the \emph{initial bound} of $y$, written $\ujb(y)$, 
%is any value that is unconditionally justified in every stable model of the program,
%i.e., $y \geq \ujb(y)$ in every stable model.
For any Boolean, we assume that $\ujb$ is fixed to $\false$.
%\pjs{Cant we just add rules for the founded initial bounds! much clearer than this!}

\noexamples{
\begin{example}
Consider a BFASP with standard variable $s$, integer 
founded variables $a,b$, 
Boolean founded variables $x$ and $y$, and the rules:
$(a \geq 0, a)$, $(b \geq 0, b)$, $(a \geq b + s, a)$,
$(b \geq 8 \leftarrow x, b)$, $(x \leftarrow \neg y \wedge (a \geq 5), x)$.
Consider an assignment $\theta$ s.t. $\theta(x) = \true$, $\theta(y) = \false$, $\theta(b) = 8$,
$\theta(s) = 9$ and $\theta(a) = 17$. The reduct of $\theta$ is the positive-CP: $a \geq 0$, $b \geq 0$, $a \geq b + 9$, $b \geq 8 \leftarrow x$,
$x \leftarrow a \geq 5$. The minimal solution that satisfies the reduct is equal to $\theta$, therefore, $\theta$ is a stable
solution of the program. Consider another assignment $\theta'$ where all values are the same as in $\theta$, but
$\theta'(s) = 3$. Then, $P^{\theta'}$ is the positive-CP: $a \geq 0$, $b \geq 0$, $a \geq b + 3$, $b \geq 8 \leftarrow x$, $x \leftarrow a \geq 5$.
The minimal solution that satisfies this positive-CP is $M$ where $M(a) = 3$, $M(b) = 0$, $M(x) = M(y) = \false$.
Therefore, $\theta'$ is not a stable solution of the program.
\end{example}
}

The focus of this paper is BFASPs where every rule is written in the form
$(y \geq f(x_1,\ldots,x_n),y)$. 
Recall that we consider the domains of Boolean variables to be ordered such that $\true > \false$. So for example, 
an ASP rule such as $a \leftarrow b \wedge c$ can equivalently be written as: $a \geq f(b, c)$ where $f$ is a Boolean that returns the value of $b \wedge c$. 
$f(x_1,\ldots, x_n)$ is essentially an expression tree where the leaf nodes are the
variables $x_1, \ldots, x_n$. 
%A variable is \emph{terminal} in an expression if it appears as a direct descendant of the root.

\noexamples{
\begin{example}
\label{ex:ast}
The function $f(x_1, \ldots, x_5) = x_1 + min(x_2, x_3 - x_4) - (x_5)^2$ can be described by the tree given below.
%Only the variable $x_1$ is terminal.

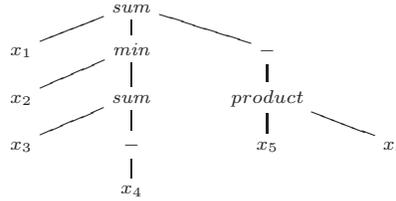
\begin{displaymath} \scriptsize
\xymatrixrowsep{0.5pc} \xymatrix{ 
	  & \sm \ar@{-}[dl] \ar@{-}[d] \ar@{-}[dr] &                                  &     \\
x_1 & \mn \ar@{-}[dl] \ar@{-}[d]             & - \ar@{-}[d]                     &     \\
x_2 & \sm \ar@{-}[dl] \ar@{-}[d]             & product \ar@{-}[d] \ar@{-}[dr] &     \\
x_3 & - \ar@{-}[d]                           & x_5                              & x_5   \\
    & x_4                                    &                                  &
}
\end{displaymath}
\end{example}
}

The \emph{local dependency graph} for a BFASP $P$ is defined over founded variables. For each rule 
$r = (y \geq f(x_1, \ldots, x_n), y)$,  there is an edge
from $y$ to all founded $x_i$. 
Each edge is marked increasing, decreasing, or non-monotonic, depending
on whether $f$ is increasing, decreasing, or non-monotonic in $x_i$.
A BFASP is \emph{locally valid} iff no edge within an SCC is marked non-monotonic. 
A program is \emph{locally stratified} if all the edges between any two nodes in the same component are marked increasing.
\noexamples{For example, if $x$ and $y$ are in the same SCC, then 
$y \geq sin(x_1)$ where $x_1$ has initial domain ($-\infty$, $\infty$) is not locally valid 
since the $sin$ function is not monotonic over this domain, but $y \geq sin(x_1)$ where $x_1$ has initial domain $[0, \pi / 2]$ is valid.}

\subsection{Non-ground BFASPs}

A \emph{non-ground BFASP} is a BFASP where sets of variables are grouped
together in variable arrays, and sets of ground rules are represented by
non-ground rules via universal quantification over index variables. For
example, if we have arrays of variables $a, b, c$, then we can represent the
ground rules: $(a[1] \geq b[1] + c[1], a[1]), (a[2] \geq b[2] + c[2], a[2]),
(a[3] \geq b[3] + c[3], a[3])$ by
$\forall i \in [1,3]: (a[i] \geq b[i] + c[i], a[i])$. Variables can be
grouped together in arrays of any dimension and non-ground BFASP rules have
the following form: $\forall \bar{i} \in \bar{D} \where con(\bar{i}):
(y[l_0(\bar{i})] \geq f(x_1[l_1(\bar{i})], \ldots, x_n[l_n(\bar{i})]),
y[l_0(\bar{i})])$, where $\bar{i}$ is a set of index variables $i_1, \ldots,
i_m$, $\bar{D}$ is a set of domains $D_1, \ldots, D_m$, $con$ is a \emph{constraint} over the index variables which
constrains these variables, $l_0, \ldots, l_n$ are functions over the index
variables which return a tuple of array indices, $y, x_1, \ldots, x_n$ are
arrays of variables and $f$ is a function over the $x_i$ variables.
Let $gen(r) \equiv \bar{i} \in \bar{D} \wedge con(\bar{i})$ denote the 
\emph{generator constraint} for a non-ground rule $r$.
Note that we require the generator constraint in each rule to constrain the
index variables so that $f$ is always defined.

\newcommand{\rch}{\mathit{rch}}
\newcommand{\node}{\mathit{node}}

Variable arrays can contain either founded variables, standard variables, or
parameters (which can simply be considered fixed standard variables),
although all variables in a variable array must be of the same type. 
Note that the array names in our notation correspond to predicate names in
standard ASP syntax, and our index variables correspond to ASP ``local
variables.'' 
%Also, in standard ASP syntax, the generator constraint is often
%put in the rule body instead of being made explicit. We make them explicit
%so that we do not have to mix Boolean conditions with arithmetic functions
%in the rule body. For example, an ASP rule: $\rch(X) \leftarrow
%\rch(Y), e(X, Y), \node(X), \node(Y)$ would be written in our syntax
%as $\forall x \in \Node, y \in \Node \text{~where~} e[x, y]: (\rch[x] \leftarrow \rch[y], \rch[x])$.
Given a non-ground rule $r$, let $\ground(r)$ be the set of ground rules obtained by substituting all possible values of the index variables that satisfy $gen(r)$ into the quantified expression. Similarly given a non-ground BFASP $P$, let $\ground(P)$ be the grounded BFASP that contains the grounding of all its rules and constraints. 
The \emph{predicate dependency graph}, validity and stratification are
defined similarly for array variables and non-ground rules as the local
dependency graph, local validity and local stratification respectively are defined for ground variables and ground rules.
All our subsequent discussion is restricted to valid BFASPs.
%Note that similarly to how local stratification is a tighter condition for logic programs,
%local validity is also a tighter condition for BFASPs.

\section{Flattening}

A ground BFASP may contain constraints and 
rules whose expressions are not \emph{flat}, i.e.,
they are expression trees with height greater than one. 
Such expressions are not supported by constraint solvers and
we need to flatten these expressions to primitive forms.
We omit consideration of flattening constraints since this is the same as in
standard CP~\cite{mznfn}.
\noexamples{
Consider the expression tree in Example \ref{ex:ast}, 
if it were a constraint, 
we would introduce variables $i_1,\ldots,i_5$ to decompose the 
given function into the following set of equalities: $f = x_1 + i_1 + i_2,
i_1 = min(x_2, i_3), i_3 = x_3 + i_4, i_4 = -x_4, i_2 = -i_5, i_5 = x_5
\times x_5$.
} 
It can be shown that the standard CP flattening approach in which a subexpression is replaced with a standard variable and a constraint is added that equates the introduced variable with the subexpression, does \emph{not} preserve stable model semantics. 
%If a standard variable is introduced in order to represent a subexpression containing founded variables, the stable solutions of the program may change.
\noexamples{
\begin{example}
\label{ex:flat}
Consider a BFASP with rules: $(x_1 \geq max(x_2, x_3) - 2, x_1), (x_2 \geq x_1 + 1, x_2)$, $(x_3 \geq x_1 + 2, x_3)$, $(x_1 \geq 3, x_1)$ where $x_1, x_2, x_3$ are all founded variables. The only stable solution of this program is $x_1 = 3, x_2 = 4, x_3 = 5$. Suppose we introduced a standard variable $i_1$ to represent the subexpression $max(x_2, x_3)$, so that the first rule in the program is replaced by: $(x_1 \geq i_1 - 2, x_1)$ and $i_1 = max(x_2, x_3)$. Now, due to the introduction of the standard variable $i_1$, the new program has many new spurious stable solutions such as $i_1 = 6, x_1 = 4, x_2 = 5, x_3 = 6$.
% or $i_1 = 7, x_1 = 5, x_2 = 6, x_3 = 7$.
\end{example}
}
To preserve the stable model semantics, it is necessary to use introduced \emph{founded} variables to represent subexpressions containing founded variables. 
We now describe the central result used in our flattening algorithm.

\newcommand{\thmflattening}{
Let $P$ be a BFASP containing a rule $r = (y \geq f_1(x_1, \ldots, x_k, f_2(x_{k+1}, \ldots, x_n)), y)$ where $f_1$ is increasing in the argument where $f_2$ appears, and where if a variable occurs among both $x_1, \ldots, x_k$ and $x_{k+1}, \ldots, x_n$, then $f_1$ and $f_2$ have the same monotonicity w.r.t. it. Let $P'$ be $P$ with $r$ replaced by the two rules:
$r_1 = (y \geq f_1(x_1, \ldots, x_k, y'), y)$ and $r_2 = (y' \geq f_2(x_{k+1}, \ldots, x_n), y')$ 
where $y'$ is an introduced founded variable. Then the stable solutions of $P'$ restricted to the variables of $P$ are equivalent to the stable solutions of $P$.
}
\begin{theorem}
\label{thm:flattening}
\thmflattening
\end{theorem}
%Theorem~\ref{thm:flattening} tells us that we can preserve the stable model semantics by introducing a founded variable to represent subexpressions containing founded variables. 
As a corollary, if $f_1$ is decreasing in the argument where $f_2$ appears, we can replace $f_2$ by a founded variable $-y'$ and add the rule $(y' \geq -f_2(x_k, \ldots, x_n), y')$ instead. 
Not all valid rule forms are supported by Theorem~\ref{thm:flattening}, because we require that multiple occurrences of the same variable in the expression must have the same monotonicity w.r.t. the root expression.
Note that if a subexpression does not contain any founded variables at all, i.e., only contains standard variables, parameters or constants, then a standard CP flattening step is sufficient.
Let us now describe our flattening algorithm $\textsf{flat}$ for ground BFASPs and later extend it to non-ground BFASPs.
%The flattening algorithm, formalized as the procedure $\textsf{flat}$, works as follows. 
We put all the rules and constraints of the program in sets $R$ and $T$ respectively. For every rule $r = (y \geq f(e_1, \ldots, e_n), y) \in R$, where $f$ is the top level function in that rule, and $e_1, \ldots, e_n$ are the expressions which form $f$'s arguments, we call $\textsf{flatRule}$ which works as follows. If there is some $e_i$ which is not a terminal, i.e., not a constant, parameter or variable, then we have two cases. If $e_i$ does not contain any founded variables, we simply replace it with standard variable $y'$ and add the constraint $y' = e_i$ to $T$. Otherwise, we apply the transformation described in Theorem~\ref{thm:flattening}. 
%At the end of $\textsf{flatRule}$, all arguments of $f$ are terminals. 
After $\textsf{flatRule}$, we simplify $r$ as much as possible through the subroutine \textsf{simplify}, e.g., by getting rid of double negations, pushing negations inside the expressions as much as possible etc. 
%We flatten all rules until $R$ is empty. 
Finally, we flatten all the constraints in $T$ using the standard CP flattening algorithm \textsf{cp\_flat} as described in~\cite{mznfn}. 
Since we replace all decreasing subexpressions by negated introduced variables and simplify expressions by pushing negations towards the variables, 
we handle negation through simple rule forms like $(y \geq -x, y)$, 
\noexamples{$(y \geq \frac{1}{x}, y)$,} $(y \geq \neg x, y)$ etc.

\noexamples{
\begin{example}
Consider the rule: $(y \geq x_1 + min(x_2, x_3 - x_4) - (x_5)^2,
y)$ where $x_1, x_2, x_5$ are founded and $x_3, x_4$ are standard variables. 
Using our flattening algorithm, we can break the rule into:
%The RHS of the initial rule is increasing in the
%subexpression $min(x_2, x_3 - x_4)$ and decreasing in $(x_5)^2$, so we can
%replace it with founded variables $i_1$ and $i_2$, 
%and ub-founded variable $i_2$\pjs{WE  CANNOT introduce ub-founded variables!} $i_2$ and break it into: 
$(y \geq x_1 + i_1 + i_2, y)$, $(i_1 \geq min(x_2, x_3 - x_4), i_1)$, $(i_2 \geq -(x_5)^2, i_2)$ where $i1,i2$ are founded variables. The rule $(i_1 \geq min(x_2, x_3 - x_4), i_1)$
is further flattened to $(i_1 \geq min(x_2, i_3), i_1)$ and a constraint $i_3 = x_3 - x_4$ where $i_3$ is a standard variable.
%requires further flattening.
%However, the subexpression $x_3 - x_4$ only contains standard variables, so we only need to introduce a standard variable $i_3$ and break it into: $(i_1 \geq min(x_2, i_3), i_1)$ and $i_3 = x_3 - x_4$.
\end{example} 
}

\begin{figure}[t]
\scriptsize
\raggedright
\begin{minipage}[t]{0.5\textwidth}
\begin{tabbing}
xx \= xx \= xx \= xx \= xx \= xx \= xx \= \kill
\textsf{flat($P$)} \\
\> $P_{flat}$ := $\emptyset$ \\
\> $R$ := $\rules(P)$ \\
\> $T$ := $\constraints(P)$ \\
\> \bfor ($r \in R$) \\
\> \> $R$ := $R \setminus \{r\}$ \\
\> \> \textsf{flatRule($r,R,T$)} \\
\> \> $r$ := \textsf{simplify($r$)} \\
\> \> $P_{flat} \cupeq \{r\}$ \\
\> \bfor ($c \in T$) $P_{flat} \cupeq $ \textsf{cp\_flat}($c$) \\
\> \breturn{$P_{flat}$} \\
\end{tabbing}
\end{minipage}
\begin{minipage}[t]{0.5\textwidth}
\begin{tabbing}
xx \= xx \= xx \= xx \= xx \= xx \= xx \= \kill
\textsf{flatRule($r = (y \geq f(e_1, \ldots, e_n), y), R, T$)} \\
\> \bfor(each non-terminal $e_i$) \\
\> \> \bif($e_i$ does not contain founded vars) \\
\> \> \> replace $e_i$ with standard var $y'$ in $r$ \\
\> \> \> $T \cupeq \{y' = e_i\}$ \\
\> \> \belseif($f$ is increasing in $e_i$) \\
\> \> \> replace $e_i$ with founded var $y'$ in $r$ \\
\> \> \> $R \cupeq \{(y' \geq e_i, y')\}$ \\
\> \> \belseif($f$ is decreasing in $e_i$) \\
\> \> \> replace $e_i$ with founded var $-y'$ in $r$ \\
\> \> \> $R \cupeq \{(y' \geq -e_i, y')\}$ \\
\end{tabbing}
\end{minipage}
\vspace*{-5mm}
\end{figure}

The algorithm can be extended to non-ground rules by defining the index set of the introduced variables
to be equal to the domain of index variables as given in the generator of the rule in which they
replace an expression. Moreover, the generator expression of an intermediate rule stays the same as that of
the original rule from which it is derived. 
%Establishing initial domains and bounds on introduced
%variables is the same as it is for the ground counterpart.

%\begin{example}
%\label{ex:max}
%Consider a nonground BFASP $P$ with arrays of founded variables, $a,b,c,d$,
%all with index sets $[1,100]$ and initial bounds of $\minf$, and a rule $r=(c,a[i])$
%where $c$ is:
%$$
%\begin{array}{l}
%\forall i \in [2.30] \where i \bmod 2 = 0: \\
%a[i] \geq b[i-1] + \mx(c[2i], d[i+1])
%\end{array}
%$$
%After flattening, we get the following rules and an intermediate variable $y_1$.
%$$
%\begin{array}{l}
%\gen(r): a[i] \geq b[i-1] + y_1[i] \\
%\gen(r): y_1[i] \geq \mx(c[2i], d[i+1]) \\
%\end{array}
%$$
%The index set for $y_1$ is $[2,30]$ and its initial bound is equal to $\minf$.
%\end{example}

\section{Grounding}
\label{sec:grounding}

ASP grounders keep track of variables that have been created and instantiate further rules based on that. 
For example, if the variables $b$ and $c$ have been created, then the rule $a \leftarrow b \wedge c$
justifies a bound on $a$ and therefore, must be included in the final program. The justification of
all positive literals in a rule potentially justify its head. However, for a rule, if any one positive variable
in its body does not have any rule supporting it, then that rule can safely be ignored until a justification for that variable
has been found. In case a justification is never found for that variable, then the rule is \emph{useless}, i.e.,
excluding the rule from the program does not change its stable solutions.

We propose a simple grounding algorithm for non-ground BFASPs
which can be implemented by simply maintaining a set of ground rules and variables as done in ASP grounders,
but which may generate useless rules in addition to all the useful ones. The idea is that for each 
variable $v$, we only keep track of whether $v$ can potentially be justified above its $\ujb$ value, 
rather than keeping track of whether it can be justified above each value in its domain. If it can be justified above its $\ujb$, 
then when $v$ appears in the body of a rule, we assume that $v$ can be justified to any possible bound for the purpose of calculating what 
bound can be justified on the head. 
\noexamples{This clearly over-estimates the bounds which can be justified on the variables, and thus the algorithm 
generates all the useful rules and possibly some useless ones.}
%\rehan{Is this missing something? Have I pruned all the unnecessary description?}

We refer to a variable $x$ as being \emph{created}, written $\created(x)$, if it can go above its $\ujb$
value. More formally, $\created(x)$ is a founded Boolean with a rule:
$\created(x) \leftarrow x > \ujb(x)$.
While that is how we define $\created(x)$, we do not explicitly have a variable
$\created(x)$ or the above rule in our implementation. 
Instead, we implement it by maintaining a set $Q$ of variables that have been created. Initially, $Q$ is empty. We recursively
look at each non-ground rule to see if the newly created variables make it
possible for more head variables to be justified above their $\ujb$
values. If so, we create those variables and add them to $Q$. In order to do
this, we need to find necessary conditions under which the head variable can
be justified above its $\ujb$. In order to simplify the presentation, we are
going to define $\ujb$ for constants, standard variables and parameters as
well. For a constant $x$, we define $\ujb(x)$ to be the value of $x$. For
parameters and 
standard variables $x$, we define $\ujb(x) = ub(x)$.\footnote{Upper and lower bounds for a parametric array can be established by simply parsing the array.}
Note that for soundness, the $\ujb$ values of founded variables only have to be correct (e.g. $\minf$ for all variables) although tighter $\ujb$
values can improve the efficiency of our algorithm. 
Table \ref{table:cond} gives a non-exhaustive list of necessary conditions for the head variable to be justified above its $\ujb$ value for different rule forms.

\begin{table} \scriptsize
\begin{tabular}{p{0.4\linewidth}|p{0.55\linewidth}}
$c$ & $\phi_r$ \\
\hline
$y \geq \sm(x_1, \ldots, x_n)$ & $(\sum_i \ujb(x_i) > \ujb(y)) \vee$ \\ 
& $((\wedge_i \ujb(x_i) > -\infty \vee \created(x_i)) \wedge (\vee_i \created(x_i)))$ \\
$y \geq \mx(x_1, \ldots, x_n)$ & $\vee_i (\ujb(x_i) > \ujb(y) \vee \created(x_i))$ \\
$y \geq \mn(x_1, \ldots, x_n)$ & $\wedge_i (\ujb(x_i) > \ujb(y) \vee \created(x_i))$ \\
$y \geq product(x_1, \ldots, x_n)$ where $\wedge_i x_i > 0$ & $\prod_i \ujb(x_i) > \ujb(y) \vee (\vee_i \created(x_i)))$ \\
$y \geq x \leftarrow r$  & $\created(r) \wedge (\ujb(x) > \ujb(y) \vee \created(x))$ \\
$y \leftarrow x \geq 0$ & $\ujb(x) \geq 0 \vee \created(x)$ \\
$y \leftarrow \wedge_{i} x_i$ & $\wedge_i \created(x_i)$ \\
$y \leftarrow \vee_{i} x_i$ & $\vee_i \created(x_i)$ \\
$y \geq -x$ & $-ub(x) > \ujb(y)$ \\
$y \leftarrow \neg x$ & $\true$ \\
$y \geq 1/x$ where $x > 0$ & $1/-ub(x) > \ujb(y)$ \\
\end{tabular}\caption{\textsl{Grounding conditions for rule $r = (c,y)$}}
\label{table:cond}
\end{table}

Let us now make a few observations about the conditions given in Table \ref{table:cond}. A key point is that for many
rule forms $\phi_r$ can evaluate to $\true$, even without any variable in the body getting created.
%For example, for $\mx$, even if one variable
%has a $\ujb$ value that is greater than the $\ujb$ of the head, the rule needs to be grounded completely.
All such rules that evaluate to true give us a starting point for initializing $Q$ in our implementation. 
%For an ASP program, this means initializing $Q$ with facts.
The linear case ($\sm$) deserves some explanation. It is made up of two disjuncts, the first of which is an evaluation of
the initial condition, i.e., whether the sum of $\ujb$ values of all variables is greater than the $\ujb$ of 
the head. If this condition is true, then the rule needs to be grounded unconditionally. If this is false, then the second
disjunct becomes important. The second disjunct itself is a conjunction of two more conditions. The first one says
that all variables must be greater than $\minf$ in order for the rule to justify a finite value on the head.
In the case where all variables already have a finite $\ujb$, the second conjunct says
that at least one of them must be created for the rule to be grounded 
\noexamples{(given the initial condition failed)}
. %Note that this condition becomes redundant if at least one of the variables has a $\ujb$ of $\minf$.
Finally, observe that after plugging all values of $\ujb$, all conditions given in the table simplify to one 
of the following four forms: $\true$, $\false$, $\vee_i \created(x_i)$ or $\wedge_i \created(x_i)$.
Note that the grounding conditions are significantly more sophisticated than the simple conjunctive condition for 
normal rules. More specifically, after simplification, we can get a disjunctive condition which has no analog in ASP.

\noexamples{
\begin{example}
Consider a BFASP with the following two non-ground rules:$\forall i \in [1,10]: a[i] \geq b[i] + x[i]$ and
$\forall i \in [1,10]: x[i] \geq \mn(c[i],d[i])$.
Say $\ujb(a) =5$, $\ujb(b)=2$, $\ujb(c)=7$, $\ujb(d)=1$
and $\ujb(x) = 1$. For the first rule, the initial condition evaluates to false. Moreover, since both $b$ and $x$ have $\ujb$
greater than $\minf$, we get $\created(b[i]) \vee \created(x[i])$.
For the second rule, since $\ujb(c[i]) > \ujb(x[i])$ and $\ujb(d[i])$ is not greater than $\ujb(x[i])$, we get the condition: $\created(d[i])$.
\end{example}
}

\begin{figure}[t]
\scriptsize
\raggedright
\begin{minipage}[t]{0.4\textwidth}
\begin{tabbing}
xx \= xx \= xx \= xx \= xx \= xx \= xx \= xx \= xx \= xx \= \kill
\textsf{createCPs($P$)} \\
\> \bfor ($r \in \rules(P) : \phi_r = \bigwedge\limits_{i=1}^n \created(x_i[\bar{l_i}])$) \\
\> \> $cp[r]$ := $\true$ \% new constraint program \\
\> \> $cp[r]$ := $cp[r] \wedge gen(r)$ \\
\> \> \bfor ($i \in 1 \ldots n$) \\
\> \> \> $set[r, i]$ := $\emptyset$ \\
\> \> \> $cp[r]$ := $cp[r] \wedge \bar{l_i} \in\,\ll\!\! set[r, i]\!\!\gg$ \\
\> \bfor ($r \in \rules(P) : \phi_r = \bigvee\limits_{i=1}^n \created(x_i[\bar{l_i}])$) \\
\> \> \bfor ($i \in 1 \ldots n$) \\
\> \> \> $cp[r, i]$ := $\true$ \% new constraint program \\
\> \> \> $cp[r, i]$ := $cp[r, i] \wedge gen(r)$ \\
\> \> \> $set[r, i]$ := $\emptyset$ \\
\> \> \> $cp[r, i]$ := $cp[r, i] \wedge \bar{l_i} \in\,\ll\!\! set[r, i]\!\!\gg$\\
\end{tabbing}
\end{minipage}
\begin{minipage}[t]{0.5\textwidth}
\begin{tabbing}
xx \= xx \= xx \= xx \= xx \= xx \= xx \= \kill
\textsf{ground($P$)} \\
\> $C := \{ \textsf{groundAll($c$)} : c \in \constraints(P) \}$ \\
\> $R' := \{ \textsf{groundAll($r$)} : r \in \rules(P) : \phi_r = \true \}$ \\
\> \bwhile ($R' \neq \emptyset$) \\
\> \> $H := \textsf{heads($R'$)}$ \\
\> \> $Q \cupeq H$ \\
\> \> $R' := \emptyset$ \\
\> \> \bfor ($r \in \rules(P) : H \cap \vars(\phi_r) \neq \emptyset$) \\
\> \> \> \bif ($\phi_r = \bigwedge\limits_{i=1}^n \created(x_i[\bar{l_i}]) \vee
	\phi_r = \bigvee\limits_{i=1}^n \created(x_i[\bar{l_i}])$) \\
\> \> \> \> \bfor ($i \in 1 \ldots n$) \\
\> \> \> \> \> $dom := \{\bar{m} \mid x[\bar{m}] \in Q\}$ \\
\> \> \> \> \> $set[r, i]$ := $dom \setminus set[r, i]$ \\
\> \> \> \> \> \bif ($\phi_r$ is conj) $R' \cupeq \textsf{search}($cp[r]$) \setminus R$ \\
\> \> \> \> \> \bif ($\phi_r$ is disj) $R' \cupeq \textsf{search}($cp[r,i]$) \setminus R$ \\
\> \> \> \> \> $R \cupeq R'$ \\
\> \> \> \> \> $set[r, i]$ := $dom$ \\
\> \bfor ($y \in \vars(R) \cap \AFV$) $R \cupeq (y \geq \ujb(y), y)$
\end{tabbing}
\end{minipage}
%\caption{Pseudo-code for bottom-up grounding.}
\vspace*{-5mm}
\end{figure}

We are now ready to present the main bottom-up grounding algorithm.
Logically, our grounding algorithm starts with $\ujb(x)$ for all $x$, adds $(x \geq ujb(x), x)$ to the program
and then finds all the ground rules that are not made redundant by these rules. 
\textsf{createCPs} is a preprocessing step that
creates constraint programs for rules in a BFASP $P$ 
whose conditions are either conjunctions or disjunctions.
For a rule with a conjunctive condition, 
it only creates one program, while for one with a disjunctive
condition, it creates one constraint program for 
each variable in the condition. 
Each program is initialized with the $\gen(r)$ which defines the variables
and some initial constraints given in the where clause in the generator of non-ground rule.
Furthermore, for each array literal in $\phi_r$, a constraint is posted
on its literal (which is a function of index variables in the rule), 
to be in the domain given by the \emph{current} 
value of the $\set$ variable (the reason
for the Quine quotes) which is initially set to empty.
\textsf{ground} is called after preprocessing.
$Q$ and $R$ are sets of ground variables and rules respectively. 
\textsf{groundAll} is a function that grounds a non-ground rule or constraint completely,
and returns the set of all rules and constraints respectively. Initially, we ground all
constraints in $P$ and rules for which $\phi_r$ evaluates to true. $R'$ is a temporary
variable that represents the set of new ground rules from the last iteration.
In each iteration, we only look for non-ground rules that have some variable in their
conditions that is created in the previous iteration.
\textsf{heads} takes a set of ground rules as its input and returns their heads.
In each iteration, through $Q$, we manipulate the $\set$ constraint to get new rule instantiations. 
For each variable in the clause, we make $\set$ equal to the new index values created for that
variable. For both the conjunctive and the disjunctive
case, this optimization only tries out new values of recently created variables to instantiate
new rules. 
\textsf{search} takes a constraint program as its input, finds all its solutions,
instantiates the non-ground rule for each solution, and returns the set of these ground rules.
After creating new rules due to the new values in $\set$, we make it equal
to all values of the variable in $Q$. The fixed point calculation stops when
no new rules are created. Finally, for every founded variable $y$, we add $(y \geq \ujb(y), y)$
as a rule so that if the $\ujb$ relied on some rules that were ignored during grounding,
then this ensures that $\ujb(y)$ is always justified.

\section{Magic set transformation}

Let us first define the \emph{query} of a BFASP.
To build the query $Q$ for a BFASP $P$, 
we ground all its constraints and its objective function,
and put all the variables that appear in them in $Q$.\footnote{Technically if
  the problem has output variables, whose value will be printed, they too need
  to be added to $Q$.}
Note that our query does not have any 
free variables and only contains ground variables. 
Therefore, 
we do not need adornment strings to propagate binding information as in the original magic set technique.
The original magic set technique has three stages: adorn, generate and modify. For the reason described above,
we only describe the latter two. 

The purpose of the magic set technique is to simulate a top-down computation
through bottom-up grounding.  
%This is achieved by creating a corresponding \emph{magic variable} $m\_a$ for every variable $a$ in the original program that represents whether we care about that variable or not. 
For every variable $a$ in the original program, we create a \emph{magic variable} $m\_a$ that represents whether we care about $a$.
Additionally, there are \emph{magic rules} that specify when a magic variable should be created.
%, meaning, when we become interested in the corresponding variable.
%Consider a simple rule $(a \geq b + c + d, a)$. Let us say that $\ujb$
%of all four variables is equal to $\minf$, and we are interested in computing
%the final value of $a$. We model this as initially setting $m\_a$ to true.
%To capture that the value of $b$ is required to compute the value of $a$, we
%add a magic rule $m\_b \leftarrow m\_a$.  We can have a similar rule for
%$m\_c$, but actually, we can make the condition for deriving $m\_c$ tighter.
%If $b$ can never go higher than $\minf$, then there is no need to know the
%value of $c$ since the rule is useless until $b$ is created. Similarly, we
%are only interested in $d$ if both $b$ and $c$ are created. 
Consider a simple rule $(a \geq b + c, a)$ where $\ujb$ of all variables is equal to $\minf$.
Suppose we are interested in computing $a$, we model this by setting $m\_a$ to true.
Since $b$ is required to compute the value of $a$, we add a magic rule $m\_b \leftarrow m\_a$.
We do not care about $c$ until a finite bound on $b$ is justified (until $b$ is created),
so we generate a tighter magic rule for $c$: $m\_c \leftarrow m\_a \wedge \cre(b)$.

We can utilize the necessary conditions for a useful grounding of a rule $r$ as given by $\phi_r$.
%We already have the necessary conditions in the form of $\phi_r$ that should be satisfied before a non-ground rule can be
%instantiated to a useful ground rule. We can utilize that information for the generation of magic rules. 
Recall that after evaluating the initial conditions, $\phi_r$ reduces to true, false, a conjunction or a disjunction. 
The above generation of magic rules for the rule $(a \geq b + c, a)$ is an example of the conjunctive case.
%For a conjunction, the magic rules are the same as they are for a normal rule in the original magic set technique. 
%For example, for the above rule $r = (a \geq b + c + d, a)$, 
%$\phi_r = \created(b) \wedge \created(c) \wedge \created(d)$, as described, 
%we get the magic rules
%$m\_b \leftarrow m\_a$ and $m\_c \leftarrow m\_a \wedge \cre(b)$.
For a disjunction, the magic rules are even simpler. 
For every $\created(x)$ in the disjunction, we create the magic rule $m\_x \leftarrow m\_a$. 
Note that not all variables in the original rule appear in the condition;
some might get removed in the simplification or not be included in the original condition at all. We can
ignore them for grounding, but we are interested in their values as soon as we know that the rule can be
useful. Therefore, as soon as the magic variable for the head is created, and $\phi_r$ is satisfied,
we are interested in all the variables in the rule that do not appear in $\phi_r$.
Finally, we define the modification step for a rule $r = (y \geq f(\bar{x}), y)$, written \textsf{modify($r$)},
as changing it to $r = (y \geq f(\bar{x}) \leftarrow m_y, y)$.
The pseudo-code for generation of magic rules and modification of the original rule is given as the function \textsf{magic} that takes a rule
as its input. It adds magic rules for a rule to a set $P$.
%A magic rule in our system is defined w.r.t. to a set of ground variables $Q$. Its body is made up of membership literals
%described in the previous section, and its head is a magic variable. 
%For a magic rule $r$, $\phi_r(Q)$ is equal to its body. 
The first two if conditions handle the disjunctive and conjunctive case respectively. The \textbf{for} loop that
follows generates magic rules for variables that are not in $\phi_r$. 
%The last line is the modification of the original rule. 

The entire bottom-up calculation with magic sets is as follows.
First, create magic variables for all the variables in the program and
call \textsf{magic} for every rule in the program. If the magic rules
generated and/or the original rule after modification are not primitive expressions, flatten them. 
Then, call \textsf{ground} on the resulting program.
While grounding the constraints, build the query by including $m\_v$ in $Q$ for every ground variable $v$ 
that is in some ground constraint. After grounding, filter all the magic variables from $Q$, and magic rules from $R$.

\begin{figure}[t]
\scriptsize
\begin{tabbing}
xx \= xx \= xx \= xx \= xx \= xx \= xx \= \kill
\textsf{magic($r$)} \\
\> $a := \head(r)$ \\
\> \bif ($\phi_r = \bigvee\limits_{i=1}^n \cre(x_i)$) \\
\> \> \bfor ($i \in 1 \ldots n$) $P \cupeq \gen(r): (m\_x_i \leftarrow m\_a, m\_x_i)$ \\
\> \bif ($\phi_r = \bigwedge\limits_{i=1}^n \cre(x_i)$) \\
\> \> \bfor ($i \in 1 \ldots n$) \\
\> \> \> $b := m\_a$ \\
\> \> \> $P \cupeq \gen(r): (m\_x_i \leftarrow b, m\_x_i)$ \\
\> \> \> $b := b \wedge \cre(x_i)$ \\
\> \bfor ($v \in \vars(r) \setminus (\vars(\phi_r) \cup \{a\})$) \\
\> \> $P \cupeq \gen(r): (m\_v \leftarrow m\_a \wedge \phi_r, m\_v)$ \\
\> $\textsf{modify}(r)$
\end{tabbing}
\vspace*{-5mm}
\end{figure}

%The next example demonstrates the complete bottom-up calculation with magic sets.
\noexamples{
\begin{example}
Consider a BFASP with the following rules:
$$
\begin{array}{llll}
R1 & \forall i \in [2,30] \text{ where } i \bmod 2 = 0 :           & R2 & \forall i \in [2,30]\text{ where } i \bmod 2 = 0 : \\
   & (a[i] \geq b[i-1] + y[i], a[i])                               && (y[i] \geq \mx(c[2i], d[i+1]), y[i]) \\
R3 & \forall i \in [1,10] : (c[i] \geq 10 \leftarrow s_1[i], c[i]) & R4 & \forall i \in [1,10] : (b[i] \geq s_2[i+1], b[i]) \\
\end{array}
$$
where $a,b,c,d,y$ are arrays of founded integers with $\ujb$ of $\minf$,
$s_2$ is an array of standard Booleans and $s_1$ is an array of
standard integers with domains $(\minf, \infty)$, and the
index set of all arrays is equal to $[1,100]$.
Let us compute $\phi_r$ for each rule.
$\phi_{R_1} = \created(b[i-1]) \wedge \created(y[i])$, 
$\phi_{R_2} = \created(c[2i]) \vee \created(d[i+1])$,
and $\phi_{R_3} = \phi_{R_4} = \true$.
We get the following magic rules (a rule $(m\_y \leftarrow body, m\_y)$ is written as $m\_y \leftarrow body$
for compactness):

$$
\begin{array}{llll}
M1 & \gen(R1) : m\_b[i-1] \leftarrow m\_a[i]  & M2 & \gen(R1) : m\_y[i] \leftarrow m\_a[i] \wedge \cre(b[i-1]) \\
M3 & \gen(R2) : m\_c[2i] \leftarrow m\_y[i]   & M4 & \gen(R2) : m\_d[i+1] \leftarrow m\_y[i] \\
M5 & \gen(R3) : m\_s_1[i] \leftarrow m\_c[i]  & M6 & \gen(R4) : m\_s_2[i+1] \leftarrow m\_b[i] \\
%M1 & \gen(R1) : (m\_b[i-1] \leftarrow m\_a[i], m\_b[i-1]) \\
%M2 & \gen(R1) : (m\_y[i] \leftarrow m\_a[i] \wedge \cre(b[i-1]), m\_y[i]) \\
%M3 & \gen(R2) : (m\_c[2i] \leftarrow m\_y[i], m\_c[2i]) \\
%M4 & \gen(R2) : (m\_d[i+1] \leftarrow m\_y[i], m\_d[i+1]) \\
%M5 & \gen(R3) : (m\_s_1[i] \leftarrow m\_c[i], m\_s_1[i]) \\
%M6 & \gen(R4) : (m\_s_2[i+1] \leftarrow m\_b[i], m\_s_2[i+1]) \\
\end{array}
$$

Let us say we are given the constraint: $a[2] + a[5] \geq 10$. Processing this, we initialize $Q$ with the set $\{m\_a[2], m\_a[5] \}$.
Running \textsf{ground} procedure extends $Q$ with the following variables, the rule used to derived a variable is given
in brackets:
$m\_b[1] (M1)$, $m\_s_2[2] (M6)$, $b[1] (R4)$, $m\_y[2] (M2)$, $m\_c[4] (M3)$, 
$m\_d[3] (M4)$, $c[4] (R3)$, $m\_s_1[4] (M5)$, $y[2] (R2)$, $a[2] (R1)$.
Filtering magic rules, the following ground rules are generated during the grounding
(the $\ujb$ of variables that are not created are plugged in as constants in rules
where they appear):
$(a[2] \leftarrow b[1] + y[2], a[2])$, $(y[2] \geq \mx(c[4], \minf), y[2])$,
$(c[4] \geq 10 \leftarrow s_1[4], c[4])$ and $(b[1] \geq s_2[2], b[1])$.
It can be shown that the number of rules with exhaustive 
and bottom-up only (without magic sets) grounding is 48 and 26 respectively!
%Without magic sets transformation and only bottom-up grounding, both $R3$
%and $R4$ yield 10 ground rules each, $R2$ gives 2 ground rules (for $i \in
%\{2,4\}$), and $R1$ gives 2 ground rules as well (for $i \in \{2,4\}$).
%With exhaustive grounding, the number of rules from $R1$ to $R4$ is 48
%(14+14+10+10). The most important point to note is that increasing the range
%of index variables in the generators of rules affects both bottom-up and
%exhaustive grounding ($\ground(P)$), but has no effect on grounding with magic sets.
\end{example}
}

If a given BFASP program is unstratified, then the algorithm described above is not sound. There might be
parts of the program that are unreachable from the founded atoms appearing in the query but are inconsistent.
We refer the reader to \cite{magic_datalog_journal} for further details.
%A simple example is a program with a rule $a \leftarrow \neg a$ and a constraint $\neg b.$ There are no stable models
%of this program particularly due to the first rule which can never be satisfied, but the magic set grounding algorithm will ignore the rule
%since it is not reachable from the $b$, and therefore wrongly declare $\neg b$ as a stable model of the program. This is
%\emph{unconditional inconsistency} \cite{magic_datalog_journal}. On the other hand, a simple example of \emph{conditional inconsistency} is the following 
%program with three rules:
%%(based on an example in \cite{magic_datalog_journal}): 
%$p \leftarrow \neg q$; $q \leftarrow \neg p$; $y \leftarrow \neg y, \neg q$
%and a constraint $\neg p \vee \neg q$. Since we are only interested in $p$ and $q$, our algorithm will ignore the third rule which 
%means that both $p=\true, q=\false$ and $q=\true, p=\false$ are stable assignments, but that is clearly wrong. The second assignment 
%is not stable due to the third rule that we ignored. This is conditional dependency which means that some but not all stable models 
%of the restricted ground program are actual stable models of the complete ground program.
%The source of both types of inconsistencies is unstratified negation as we can prove that without it, we can safely
%ignore parts of the program that are not reachable from the constraints. 
We overcome this 
by including in the query all ground magic variables of all array variables that are part of a component in
the dependency graph in which there is some decreasing (negative) edge between any two of its nodes.
The following result establishes correctness of our approach.
%For the following result which establishes correctness of our approach, let $M$ be a ground BFASP produced by running the magic set transformation after including
%the unstratified parts of the program in the initial query for a given non-ground BFASP $P$.
%Let $G$ be equal to $\ground(P)$, and let $P_i$ be part of $\ground(P)$ that is not included
%in $M$.

\newcommand{\thmgrounding}{
Given a BFASP $P$, let $G$ be equal to $\ground(P)$ and $M$ be a ground BFASP produced by running the magic set transformation after including
the unstratified parts of the program in the initial query for a given non-ground BFASP $P$.
The stable solutions of $G$ restricted to the variables $vars(M)$ are equivalent to the stable solutions of $M$. That is,
if $\theta'$ is a stable solution of $G$, then $\theta'|_{vars(M)}$ is a stable solution of $M$ and
if $\theta$ is a stable solution of $M$, then there exists $\theta'$ s.t. $\theta'$ is a stable solution of $G$ and $\theta'|_{vars(M)} = \theta$.
}
\begin{theorem}
\label{theorem:ground}
\thmgrounding
\end{theorem}

\section{Experiments}

We show the benefits of bottom-up grounding and magic sets for computing
with BFASPs on a number of benchmarks: \textsl{RoadCon},  \textsl{UtilPol}
and \textsl{CompanyCon}.\footnote{All problem encodings and instances can be
  found at: \texttt{www.cs.mu.oz.au/\textasciitilde{}pjs/bound\_founded/}}
In utilitarian policies (\textsl{UtilPol}), a government decides
a set of policies to enact while minimizing the cost. Additionally, there are
different citizens and each citizen's happiness depends on the enacted policies and 
happiness of other citizens. There is a citizen $t$ whose happiness should be above a given value.
Company controls (\textsl{CompanyCon}) is a problem related to stock markets.
The parameters of the problem are the number of companies, each company's ownership
of stocks in other companies, and a source company that wants to \emph{control} a destination
company. The decision variables are the number of stocks that the source company buys
in every other company. A company $c$ controls a company $d$ if the number of stocks
that $c$ owns in $d$ plus the number of stocks that other companies that $c$ controls own
in $d$ is greater than 50 percent of total number of stocks of company $d$.
The objective is to minimize the total cost of stocks bought. All experiments were performed on a machine running Ubuntu 12.04.1 LTS
with 8 GB of physical memory and Intel(R) Core(TM) i7-2600 3.4 GHz processor. Our implementation 
extends MiniZinc 2.0 (\libmzn) and uses the solver \chuffed extended with
founded variables and rules as described in our previous work \cite{bfasp}.
Each time in the tables is the median time in seconds of 10 different instances.

\begin{table}[t] \scriptsize
\centering
\begin{tabular}{rc|rr|rr|rr}
\hline
&& \multicolumn{2}{c|}{Exhaustive} & \multicolumn{2}{c}{Bottom-up} & \multicolumn{2}{c}{Magic} \\
$N$ & SCCs & Flat  & Solve  &  Flat  &   Solve &   Flat  & Solve  \\[0.2ex]
\hline
100 & 5    & 4.25    & 3.34    & 1.37    & .64     & .27     & .04     \\ [0.5ex] 
300 & 15   & 39.02   & ---     & 4.19    & 1.25    & .41     & .07     \\ [0.5ex] 
600 & 20   & 237.97  & ---     & 19.70   & 22.56   & .83     & .96     \\ [0.5ex] 
900 & 30   & ---     & ---     & 30.44   & 127.90  & 1.17    & 4.74    \\ [0.5ex] 
1400 & 45  & ---     & ---     & 56.99   & 398.29  & 1.79    & 25.66   \\ [0.5ex] 
%1400 & 20  & ---     & ---     & 258.31  & ---     & 3.88    & 535.68  \\ [0.5ex] 

%100 & 5    & 4.64  &  2.12  &  1.52  &   0.37  &   0.32  & 0.05  \\ [0.5ex]
%300 & 15   & 37.42 &   ---  &  4.05  &   3.67  &   0.43  & 0.16  \\ [0.5ex]   
%600 & 20   & 240.61&   ---  &  20.11 &   19.81 &   0.88  & 0.93  \\ [0.5ex]
%900 & 30   & ---   &   ---  &  30.58 &   38.86 &   1.24  & 2.29 \\ [0.5ex]
%1400& 45   & ---   &   ---  &  60.61 &   ---   &   1.87  & 24.98 \\ [0.5ex]
%1400& 20   & ---   &   ---  &  266.85&   ---   &   3.99  & --- \\ [0.5ex]
\hline \\[0.2ex] 
\end{tabular}
\caption{Road Construction \textsl{RoadCon}}
\label{table:road}
\end{table}

\begin{table}[t]
\scriptsize
%\footnotesize 
\begin{minipage}{.57\linewidth}
\centering
\tabcolsep=0.05cm
\begin{tabular}{rccc|rr|rr}
\hline
\multicolumn{4}{c|}{Instance} & \multicolumn{2}{c|}{Bottom-up} & \multicolumn{2}{c}{Magic} \\
$C$ & $P$ & $C_r$  & $P_r$  &  Flat  &   Solve &   Flat  & Solve  \\[0.2ex]
\hline
50 & 100 & 5 & 5        & 2.02    & 1.90    & .48     & .01     \\ [0.5ex] 
100 & 300 & 10 & 30     & 16.62   & 91.66   & 2.97    & .07     \\ [0.5ex] 
100 & 500 & 10 & 30     & 24.78   & ---     & 4.39    & .09     \\ [0.5ex] 
250 & 350 & 105 & 105   & 83.45   & ---     & 35.16   & 18.40   \\ [0.5ex] 
250 & 400 & 110 & 110   & 88.61   & ---     & 39.32   & 452.17  \\ [0.5ex] 
300 & 400 & 125 & 150   & 140.36  & ---     & 57.09   & ---     \\ [0.5ex] 
\end{tabular}
\caption{\textsl{UtilPol}}
\label{table:utilpol}
\end{minipage}
\begin{minipage}{.4\linewidth}
\centering
\tabcolsep=0.05cm
\begin{tabular}{rc|rr|rr}
\hline
\multicolumn{2}{r|}{Instance} & \multicolumn{2}{c|}{Bottom-up} & \multicolumn{2}{c}{Magic} \\
$C$ & $C_r$ &  Flat  &   Solve &   Flat  & Solve  \\[0.2ex]
\hline
1000 & 15  & 24.27   & 5.20    & .79     & .70     \\ [0.5ex] 
%1000 & 20  & 27.08   & 6.39    & 1.00    & 1.53    \\ [0.5ex] 
%1250 & 20  & 37.05   & 12.81   & 1.07    & 5.01    \\ [0.5ex] 
1500 & 25  & 53.66   & 17.52   & 1.39    & 2.07    \\ [0.5ex] 
2000 & 35  & 94.38   & 66.81   & 2.31    & 8.75    \\ [0.5ex] 
3000 & 50  & 209.70  & 86.35   & 1.71    & 17.87   \\ [0.5ex]
3500 & 60  & ---     & ---     & 5.58    & 19.18   \\ [0.5ex] 
%4000 & 70  & ---     & ---     & 7.07    & 46.55   \\ [0.5ex] 
5000 & 80  & ---     & ---     & 9.63    & 51.45   \\ [0.5ex] 
\end{tabular}
\caption{\textsl{CompanyCon}}
\label{table:companycon}
\end{minipage}
%\label{table:utilpol_companycon}
%\caption{Utilitarian Policies \textsl{UtilPol} and Company Controls \textsl{CompanyCon}}
\end{table}

Table \ref{table:road} shows the results for \textsl{RoadCon}.
$N$ is the number of nodes, and SCCs is 
the minimum number of strongly connected components in the graph. 
%The edge probability between any two nodes in an SCC is $0.2$.
We compare exhaustive grounding (simply creating $\ground(P)$) against
bottom-up grounding, and bottom-up grounding with 
magic set transformation. A --- 
represents either the flattener/solver did not finish in 10
minutes or that it ran out of memory.
Using bottom-up grounding, the founded variables representing 
shortest paths between two nodes that are not in the same
SCC and the corresponding useless rules are not created. 
%Moreover, many useless rules for such variables are also not created.
Clearly bottom-up grounding is far superior to naively grounding
everything, and magic sets substantially improves on this.
Tables \ref{table:utilpol} and \ref{table:companycon} show the results for utilitarian policies and
company controls respectively. 
The running time for exhaustive and bottom-up for these benchmark
are similar, therefore, the comparison is only given for bottom-up vs. magic
sets. For \textsl{UtilPol}, $C$ and $P$ represent the number of citizens and policies respectively,
$C_r$ represents the maximum number of relevant citizens on which the
happiness of $t$ directly or indirectly depends and  
$P_r$ is the maximum number of policies on which the happiness of $t$ and other citizens
in $C_r$ depends. 
This is the part of the instance that is relevant to the query and the rest is ignored when magic sets are enabled. 
It can be seen
that magic sets outperform regular bottom-up grounding, especially when the
relevant part of the instance is small compared to the entire instance.
Note that when $P_r$ is small, 
the flattening time for magic sets is greater that the solving time
since the resulting set of rules is actually simple. 
This changes, however, as $P_r$ is increased.
For \textsl{CompanyCon},
$C$ is the number of total companies while $C_r$ is the maximum number of companies reachable
from the destination in the given ownership graph. 
The table shows that if $C_r$ is small
compared to $C$, magic sets can give significant advantages.
\noexamples{The unnecessary founded variables and rules can 
make solving time considerably higher if magic sets optimization is not used.}

\section{Conclusion}

Bound Founded Answer Set Programming extends ASP to
disallow circular reasoning over numeric entities. 
\noexamples{While the semantics of BFASP is a simple generalization of the semantics of
ASP, to be practically useful we must be able to model non-ground BFASPs
in a high level way.}  
In this paper, we show how we can flatten and ground a
non-ground BFASP while preserving its semantics, thus creating an executable
specification of the BFASP problem. We show that using bottom-up grounding
and magic sets transformation we can significantly improve the efficiency of
computing BFASPs. The existing magic set techniques are only defined for the normal
rule form, involving only founded Boolean variables. We have extended magic sets to BFASP, a formalism that
has significantly more sophisticated rule forms and has both standard and founded variables,
that can moreover be Boolean or numeric.
%To the best of our knowledge, this is the first paper
%that extends magic set transformation to 
%hybrid systems involving both standard and founded variables such as \clingcon \cite{clingcon} and the \idp system \cite{idp}.

\bibliography{paper}
\bibliographystyle{acmtrans}

\newtheorem{apxTheorem}{Theorem}

\appendix
\section{Proofs of theorems}
\begin{apxTheorem}
\thmflattening
\end{apxTheorem}
\begin{proof}
For a rule $s=(c,head)$, let $con(s)=c$.
By construction, $con(r) \Leftrightarrow \proj{y'}{con(r_1) \wedge con(r_2)}$ and all the other constraints in $P$ and $P'$ are identical. Also, given any assignment $\theta'$ of $P'$, since $f_1$ is increasing in the argument where $f_2$ appears, $y'$ will be left as a variable in $\red{r_1}{\theta}$. 
Consider an assignment $\theta'$ over $\vars(P')$, and let $\theta = \theta'|_{vars(P)}$.
Recall that the reduct of a program with respect to an assignment replaces all the standard variables and founded variables that are not decreasing in any rule's constraint with
its value in that assignment. Since $f_1$ and $f_2$ have the same monotonicity w.r.t. any variable common in $\{x_1, \ldots, x_k\}$ and $\{x_{k+1}, \ldots, x_n\}$, it will either
be replaced by its assignment value in both $f_1$ and $f_2$ or not be replaced at all.
Therefore, the relation $\red{r}{\theta} \Leftrightarrow \proj{y'}{\red{r_1}{\theta'} \wedge \red{r_2}{\theta'}}$ is also valid.
Furthermore, all other constraints in $P^{\theta}$ and $P'^{\theta'}$ are identical.

Suppose $\theta$ is a stable solution of $P$. Let $\theta'$ be the extension of $\theta$ to variable $y'$ s.t. $\theta'(y') = f_2(\theta(x_k), \ldots, \theta(x_n))$. Clearly, this choice of $\theta'(y')$ allows $\theta'$ to satisfy all the constraints of $P'$ and allows $\theta'|_{vars(P'^{\theta'})}$ to satisfy all the constraints of $P'^{\theta'}$. To prove that $\theta'$ is a stable solution of $P'$, we just need to show that there is no smaller solution of $P'^{\theta'}$ than $\theta'|_{vars(P'^{\theta'})}$. Since $\red{r}{\theta} \Leftrightarrow \proj{y'}{\red{r_1}{\theta'} \wedge \red{r_2}{\theta'}}$ and all other constraints in $P^{\theta}$ and $P'^{\theta'}$ are identical, $P'^{\theta'}$ must force the same lower bounds on the variables in $vars(P^{\theta})$ as $P^{\theta}$ does. Hence, none of those values can go any lower. Also, $\red{r_2}{\theta'}$ forces $y' \geq f_2(\theta(x_k), \ldots, \theta(x_n))$, and so $f_2(\theta(x_k), \ldots, \theta(x_n))$ is the lowest possible value for $y'$. Hence $\theta'|_{vars(P'^{\theta'})}$ is the minimal solution of $P'^{\theta'}$ and $\theta'$ is a stable solution of $P'$.

Suppose $\theta'$ is a stable solution of $P'$. Let $\theta =
\theta'|_{vars(P)}$. Since $con(r) \Leftrightarrow \proj{y'}{con(r_1)
  \wedge con(r_2)}$ and all the other constraints in $P$ and $P'$ are
identical, $\theta$ satisfies all the constraints in $P$. Since
$\red{r}{\theta} \Leftrightarrow \proj{y'}{\red{r_1}{\theta'} \wedge
  \red{r_2}{\theta'}}$ and all other constraints in $P^{\theta}$ and
$P'^{\theta'}$ are identical, $\theta|_{vars(P^{\theta})}$ satisfies all the
constraints in $P^{\theta}$. To prove that $\theta$ is a stable solution of
$P$, we just need to show that there is no smaller solution of $P^{\theta}$
than $\theta|_{vars(P^{\theta})}$. Since $\red{r}{\theta} \Leftrightarrow
\proj{y'}{\red{r_1}{\theta'} \wedge \red{r_2}{\theta'}}$ and all other
constraints in $P^{\theta}$ and $P'^{\theta'}$ are identical, $P^{\theta}$
must force the same lower bounds on the variables in $vars(P^{\theta})$ as
$P'^{\theta'}$ does. Hence, none of those values can go any lower,
$\theta|_{vars(P^{\theta})}$ is the minimal solution of $P^{\theta}$ and
$\theta$ is a stable solution of $P$.
\end{proof}

\begin{apxTheorem}
\thmgrounding
\end{apxTheorem}
\begin{proof}

Let us first argue about the correctness of our grounding approach presented in Section~\ref{sec:grounding}.
We can analyze each row in Table \ref{table:cond} and reason
that until the condition is satisfied, the rule can be ignored without changing the stable solutions
of the program. We only provide a brief sketch and do not analyze each case in the table. Say, e.g. for
$y \geq \mx(x_1, \ldots, x_n)$, if the condition is not satisfied, 
this means that no $x_i$ has a rule in the program that justifies a value higher than its $\ujb$, 
and no $x_i$ initially justifies a bound on $y$ that is greater than $\ujb(y)$. 
If we include a ground version of this rule in the program, then after taking 
the reduct w.r.t. any assignment, the rule can never justify any bound
on the head, and hence can safely be eliminated.

Let $P_i$ be part of $\ground(P)$ that is not included in $M$.
It can be seen from the description of magic set transformation that any variable
in $P_i$ either cannot be reached from any variable in $M$ in the dependency graph of $P$,
or can only be reached through useless rules. Since useless rules can be
eliminated as argued above,
%Section~\ref{sec:grounding}, 
we conclude that no variable in $M$ can reach any variable in $P_i$ in the dependency graph.
This obviously also holds for dependency graph of respective reduced program w.r.t. some
assignment. This means that for a given assignment $\theta'$, the minimal order computation can
first be performed on $M^{\theta'}$ which fixes all the variables in $\vars(M)$,
and then on $P_i^{\theta'}$ which fixes all the remaining variables, i.e., variables in $\vars(P) - \vars(M)$.
Combining both the minimal solutions would be the same as computing the minimal solution
for $G^{\theta'}$. This proves the first result.

For the second result, since all unstratified parts in $P$ are included in
$M$, all the intra-component edges in the dependency
graph of $P_i$ are marked increasing (positive). 
It can be shown that for such a program, once we fix all the standard variables appearing
in any rule in $P_i$, there is a unique stable solution that can be computed as the
\emph{iterated least fixpoint} of $P_i$. This is similar to the well known result
for logic programs that states that for a stratified program, the unique stable solution can be computed as
the iterated least fixpoint of the program (Corollary 2 in \cite{stable}). 
Therefore, if we are given a stable solution $\theta$
for $M$, we can extend it to $\theta'$ by fixing all the unfixed standard variables
to any value, and then computing the iterated least fixpoint, which will extend $\theta'$ over
founded variables of $P_i$, and will be a unique stable solution given the values
of all standard variables.
%It can be shown that for such a program, we can fix
%each standard variable to any value in its domain and that will result in a unique stable model. 
%Once the standard variables are fixed, the stable model of $P_i$ can be computed by iterating over each component in an order that visits all its
%lower SCCs first before visiting any component and performing the
%following computation on each component. Each lowest component is a
%positive-CP and its minimal solution is unique, compute that and fix their
%values in lower components to their value in the computed minimal model.
%As a result, the next SCCs in order become positive-CPs and so on. Once we
%have visited all the components, by combining the minimal models of all
%components, we get the unique stable model. 
\end{proof}

\end{document}